\documentclass{article}
\usepackage{graphics,amsthm,fullpage}
\usepackage{times}

\usepackage[usenames,dvipsnames]{color}
\usepackage[pdfauthor={Palash Dey},pdfnewwindow=true,colorlinks=true,linkcolor=blue,citecolor=Mahogany,pagebackref=true,breaklinks=true]{hyperref}

\usepackage{mathrsfs}
\usepackage{amssymb}
\usepackage{amsmath}
\usepackage{amsthm}
\usepackage{perpage} %the perpage package
\usepackage{array}
\usepackage{enumitem}

\usepackage{multirow}

\allowdisplaybreaks

\renewcommand{\epsilon}{\varepsilon}
\newcommand{\eps}{\varepsilon}
\newcommand{\E}[1]{\mathbb{E}\left[#1\right]}

\usepackage{amsopn}
\DeclareMathOperator*{\argmax}{arg\,max}
\DeclareMathOperator*{\argmin}{arg\,min}

\newcommand{\eqdef}{\overset{\mathrm{def}}{=\joinrel=}}

\newcommand\numberthis{\addtocounter{equation}{1}\tag{\theequation}}

\makeatletter
\newcommand{\mathleft}{\@fleqntrue\@mathmargin\parindent}
\newcommand{\mathcenter}{\@fleqnfalse}
\makeatother

\newcommand{\MOV}[1][r]{M_{#1}(\mathcal{E})}
\newcommand{\election}{\mathcal{E}=(V, C)}

\usepackage{mathtools}
\DeclarePairedDelimiter{\ceil}{\lceil}{\rceil}
\DeclarePairedDelimiter{\floor}{\lfloor}{\rfloor}

\usepackage{url}

\renewenvironment{proof}{\noindent{\em Proof:}}{ \hfill $\square$\\ }

\newenvironment{proof_sketch}{\noindent{\em Proof sketch:}}{ \hfill $\square$\\ }

\usepackage{xspace}

\newcommand{\NPshort}{\ensuremath{\mathsf{NP}}\xspace}

\newtheorem{theorem}{\bf Theorem}
\newtheorem{lemma}{\bf Lemma}
\newtheorem{corollary}{\bf Corollary}
\newtheorem{definition}{\bf Definition}
\newtheorem{claim}{\bf Claim}

\title{Estimating the Margin of Victory of an Election using Sampling}
\author{Palash Dey and Y. Narahari\\
Department of Computer Science and Automation, Indian Institute of Science, Bangalore, India\\
\{palash$\mid$hari\}@csa.iisc.ernet.in}

\begin{document}
% \sloppy

\maketitle

\begin{abstract}
The margin of victory of an election is a useful measure to capture the robustness of an election outcome. It also plays a crucial role in determining the sample size of various algorithms in post election audit, polling etc. In this work, we present efficient sampling based algorithms for estimating the margin of victory of elections.

More formally, we introduce the \textsc{$(c, \epsilon, \delta)$--Margin of Victory} problem, where given an election $\mathcal{E}$ on $n$ voters, the goal is to estimate the margin of victory $M(\mathcal{E})$ of $\mathcal{E}$ within an additive factor of $cM(\mathcal{E})+\eps n$. We study the \textsc{$(c, \epsilon, \delta)$--Margin of Victory} problem for many commonly used voting rules including scoring rules, approval, Bucklin, maximin, and Copeland$^{\alpha}.$ We observe that even for the voting rules for which computing the margin of victory is $\NPshort$-Hard, there may exist efficient sampling based algorithms, as observed in the cases of maximin and Copeland$^{\alpha}$ voting rules.
\end{abstract}

\section{Introduction}

In many real life applications, there is often a need for a set of agents to agree upon a common decision although 
they may have different preferences over the available candidates to choose from. A natural 
approach used in these situations is voting. Some prominent examples 
of the use of voting rules in the context of multiagent systems include 
collaborative filtering~\cite{pennock2000social}, personalized product selection~\cite{lu2011budgeted} etc.

In a typical voting scenario, we have a set of votes each of which is a complete ranking over a set of candidates. We also have a function called voting rule that takes as input a set of votes and outputs a candidate as the winner. A set of votes over a set of candidates along with a voting rule is called an {\em election} and the winner is called the outcome of the election. 
% The plurality voting rule is one of the most commonly used voting rules where the winner is the candidate that occurs as the top candidate in the votes more often than any other candidate. 
% By far the most commonly used voting rule is the plurality voting rule where the winner is the candidate that occurs as the top candidate in the votes more often than any other candidate does. 
% The plurality voting rule is used frequently in political elections. Other than the plurality voting rule, there are quite a few other voting rules that are also used in many applications. One such voting rule is the Borda voting rule which is used in many educational institutes in the United States (US). 
% Section\nobreakspace \ref {sec:prelim} discusses commonly used voting rules.

Given an election, one may like to know how robust the election outcome is with respect to the changes in votes~\cite{shiryaev2013elections,caragiannis2014modal,regenwetter2006behavioral}. One way to capture robustness of an election outcome is to compute the minimum number of votes that must be changed to change the outcome. This idea of robustness is captured precisely by the notion called {\em margin of victory}. The margin of victory of an election is the smallest number of votes that need to be changed to change the election outcome. In a sense, an election outcome is considered to be robust if the margin of victory is large.

% However, irrespective of the voting rule that is being used, one may like to know how robust the election outcome is with respect to the changes in votes~\cite{shiryaev2013elections,caragiannis2014modal}. One way to capture robustness of an election outcome is to compute the minimum number of votes that must be changed to change the outcome. This idea of robustness is captured precisely by the notion called {\em margin of victory}. The margin of victory of an election is the smallest number of votes that need to be changed to change the election outcome. In a sense, an election outcome is considered to be robust if the margin of victory is large.

\subsection{Motivation}\label{sec:motivation}

In addition to being interesting purely because of theoretical reasons, the margin of victory of an election plays a crucial role in many practical applications. One such example is post election audits --- methods to observe a certain number of votes (which is often selected randomly) after an election to detect an incorrect outcome. There can be a variety of reasons for an incorrect election outcome, for example, software or hardware bugs in voting machine~\cite{norden2007post}, machine output errors, use of various clip-on devices that can tamper with the memory of the voting machine~\cite{wolchok2010security}, human errors in counting votes. Post election audits have nowadays become common practice to detect problems in electronic voting machines in many countries, for example, the US. As a matter of fact, at least thirty states in the US have reported such problems by 2007~\cite{norden2007post}. Most often, the auditing process involves manually observing some sampled votes.  Researchers have subsequently proposed various {\em risk limiting auditing} methodologies that not only minimize the cost of manual checking, but also limit the risk of making a human error by sampling as few votes as possible~\cite{stark2008conservative,stark2008sharper,stark2009efficient,sarwate2011risk}. The sample size in a risk limiting audit critically depends on the margin of victory of the election. 

Another very important application where the margin of victory plays an important role is polling. In polling, the pollster samples a certain number of votes from the population and predicts the outcome of the underlying election based on the outcome of the election on the sampled votes. One of the most fundamental questions in polling is: how many votes should be sampled? It turns out that the {\em sample complexity} in polling too crucially depends on the  margin of victory of the election from which the pollster is sampling~\cite{canetti1995lower,sampling}. The number of samples used in an algorithm is called the sample complexity of that algorithm. As the above discussion shows, computing the margin of victory of an election is often a necessary task in many practical applications. However, one cannot observe all the votes in many applications including the ones discussed above. For example, in a survey or polling, one cannot first observe all the votes to compute the margin of victory and then sample a few votes based on the computed margin of victory. Hence, one often needs a {\it ``good enough''} estimate of the margin of victory by observing a few votes. We, in this work, precisely address this problem: estimate the margin of victory of an election by sampling as few votes as possible.

% \cite{norden2007post,wolchok2010security,stark2008conservative,stark2008sharper,stark2009efficient,sarwate2011risk}.  

\subsection{Our Contributions}

Let $n$ be the number of votes, $m$ the number of candidates, $r$ any voting rule. We introduce and study the following computational  problem in this paper\footnote{Throughout this section, we use standard terminlogy from voting theory. For formal definitions, refer to Section\nobreakspace \ref {sec:prelim}.}: 
\begin{definition}\textsc{(\textsc{$(c, \epsilon, \delta)$--Margin of Victory (MoV)})}\\\label{def:prob}
 Given a $r$-election $\mathcal{E}$, determine $\MOV$, the margin of victory of $\mathcal{E}$ with respect to $r$, within an additive error of at most $c\MOV + \epsilon n$ with probability at least $1-\delta$. The probability is taken over the internal coin tosses of the algorithm.
\end{definition}
We call the parameter $(c,\eps)$ in \MakeUppercase definition\nobreakspace \ref {def:prob} the approximation factor of the problem.
The notion of approximation in \MakeUppercase definition\nobreakspace \ref {def:prob} is a hybrid of what are classically known as additive and multiplicative approximations (see~\cite{vazirani2001approximation}). However, \MakeUppercase corollary\nobreakspace \ref {cor:lb} shows that, there does not exist any estimator with sample complexity independent of $n$ and achieves $o(n)$ additive approximation. Again, we can not hope to have an estimator with sample complexity independent of $n$ that guarantees good multiplicative approximation ratio since there exist elections with margin of victory only one. This justifies the problem formulation in \MakeUppercase definition\nobreakspace \ref {def:prob}. 

Our goal here is to solve the \textsc{$(c, \epsilon, \delta)$--MoV} problem with as few sample votes as possible. Our main technical contribution is to come up with efficient sampling based {\em polynomial time} randomized algorithms to solve the \textsc{$(c, \epsilon, \delta)$--MoV} problem for various voting rules. Each sample reveals the entire preference order of the sampled vote. The specific contributions of this paper are summarized in Table\nobreakspace \ref {table:margin of victory_summary}. 

\begin{table}[htbp]
\begin{minipage}{\textwidth} 
% \centering
  \resizebox{\linewidth}{!}{ 
  %for adjusting height of table cells
  \renewcommand{\arraystretch}{1.77}
\begin{tabular}{|>{\centering\arraybackslash}c|>{\centering\arraybackslash}c|>{\centering\arraybackslash}c| }\hline
  \textbf{\multirow{1}{*}{Voting Rule}}	& \multicolumn{2}{c|}{\textbf{\multirow{1}{*}{Sample complexity}}} \\\hline\hline
  \multirow{1}{*}{Scoring rules}	& \multirow{1}{*}{$(\frac{1}{3}, \epsilon, \delta)$--MoV, $\frac{12}{\epsilon ^2}\ln \frac{2m}{\delta}$ [\MakeUppercase theorem\nobreakspace \ref {thm:scr}]} & \multirow{6}{*}{\vtop{\hbox{\strut $(c^\prime, \epsilon, \delta)$--MoV$^\dagger$,}\hbox{\strut $\frac{(1-c)^2}{36\epsilon^2}\ln \left(\frac{1}{8e\sqrt{\pi}\delta}\right)$,}\hbox{\strut [\MakeUppercase corollary\nobreakspace \ref {cor:lb}]}} } \\\cline{1-2}
  \multirow{1}{*}{$k$-approval}	& \multirow{1}{*}{$(0, \epsilon, \delta)$--MoV, $\frac{12}{\epsilon^2}\ln \frac{2k}{\delta}$, [\MakeUppercase theorem\nobreakspace \ref {thm:kapp}]} &  \\\cline{1-2}
  \multirow{1}{*}{Approval}	& \multirow{1}{*}{$(0, \epsilon, \delta)$--MoV, $\frac{12}{\epsilon^2}\ln \frac{2m}{\delta}$, [\MakeUppercase theorem\nobreakspace \ref {thm:app}]} &  \\\cline{1-2}
  \multirow{1}{*}{Bucklin}	& \multirow{1}{*}{$(\frac{1}{3}, \epsilon, \delta)$--MoV, $\frac{12}{\epsilon^2}\ln \frac{2m}{\delta}$, [\MakeUppercase theorem\nobreakspace \ref {thm:bucklin}]}	& 	\\\cline{1-2} 
  \multirow{1}{*}{Maximin}	& \multirow{1}{*}{$(\frac{1}{3}, \epsilon, \delta)$--MoV, $\frac{24}{\epsilon^2}\ln \frac{2m}{\delta}$, [\MakeUppercase theorem\nobreakspace \ref {thm:maximin}]} &  \\\cline{1-2}
  \multirow{1}{*}{Copeland$^\alpha$}	& \multirow{1}{*}{$\left(1-O\left(\frac{1}{\log m}\right), \epsilon, \delta\right)$--MoV, $\frac{96}{\epsilon^2}\ln \frac{2m}{\delta}$, [\MakeUppercase theorem\nobreakspace \ref {thm:copeland}]}  &		\\\hline
 \end{tabular}}
  \caption{\normalfont Sample complexity for the \textsc{$(c, \epsilon, \delta)$--MoV} problem for various voting rules. $\dagger$The result holds for any $c^\prime \in [0,1).$}
  \label{table:margin of victory_summary}
 \end{minipage}
%  }
\end{table} 

Table\nobreakspace \ref {table:margin of victory_summary} shows a practically appealing positive result- {\em the sample complexity of all the algorithms presented here is independent of the number of voters}. We also present lower bounds on the sample complexity of the \textsc{$(c, \epsilon, \delta)$--MoV} problem for all the common voting rules which matches with the upper bounds when we have a constant number of candidates. Moreover, the lower and upper bounds on the sample complexity match exactly for the $k$-approval voting rule irrespective of number of candidates, when $k$ is a constant. The specific contributions of this paper are as follows.

\begin{itemize}%[noitemsep,leftmargin=0cm,itemindent=.5cm,labelwidth=\itemindent,labelsep=0cm,align=left]
 \item We show a sample complexity lower bound of $\Omega(\frac{1}{\epsilon ^2}\log \frac{1}{\delta})$ for the \textsc{$(c, \epsilon, \delta)$--MoV} problem for all the commonly used voting rules, where $c\in [0,1)$ (\MakeUppercase theorem\nobreakspace \ref {thm:lb} and \MakeUppercase corollary\nobreakspace \ref {cor:lb}).
 \item We show a sample complexity upper bound of $O(\frac{1}{\epsilon^2}\log \frac{m}{\delta})$ for the \textsc{$(\frac{1}{3}, \epsilon, \delta)$--MoV} problem for arbitrary scoring rules (\MakeUppercase theorem\nobreakspace \ref {thm:scr}). However, for a special class of scoring rules, namely, the $k$-approval voting rules, we have a sample complexity upper bound of $O(\frac{1}{\epsilon ^2}\log \frac{k}{\delta})$ for the \textsc{$(0, \epsilon, \delta)$--MoV} problem (\MakeUppercase theorem\nobreakspace \ref {thm:kapp}).
\end{itemize}

One key finding of our work is that, there may exist efficient sampling based polynomial time algorithms for estimating the margin of victory, even if computing the margin of victory is $\NPshort$-Hard for a voting rule~\cite{xia2012computing}, as observed in the cases of maximin and Copeland$^\alpha$ voting rules.  

\subsection{Related Work and Discussion}

Magrino et al.~\cite{magrino2011computing} presents approximation algorithms to compute the margin of victory for the instant runoff voting (IRV) rule. Cary~\cite{cary2011estimating} provides algorithms to estimate the margin of victory of an IRV election. Xia~\cite{xia2012computing} presents polynomial time algorithms for computing the margin of victory of an election for various voting rules, for example the scoring rules, and proved intractability results for several other voting rules, for example the maximin and Copeland$^\alpha$ voting rules. Endriss et al.~\cite{endriss2014margin} computes the complexity of exact variants of the margin of victory problem for Schulze, Cup, and Copeland voting rules. However, all the existing algorithms to either compute or estimate the margin of victory {\em need to observe all the votes,} which defeats the purpose in many applications including the ones discussed in~Section\nobreakspace \ref {sec:motivation}. We, in this work, show that we can estimate the margin of victory for many common voting rules quite accurately by sampling a few votes only. Moreover, the accuracy of our estimation algorithm is good enough for many practical scenarios. For example, Table\nobreakspace \ref {table:margin of victory_summary} shows that it is enough to select only $3600$ many votes uniformly at random to estimate $\frac{\text{MoV}}{n}$ of a plurality election within an error of $0.1$ with probability at least $0.99$, where $n$ is the number of votes. We note that in all the sampling based applications discussed in~Section\nobreakspace \ref {sec:motivation}, the sample size is inversely proportional to $\frac{\text{MoV}}{n}$~\cite{canetti1995lower} and thus it is enough to estimate $\frac{\text{MoV}}{n}$ accurately.

The margin of victory problem is the same as the optimization version of the destructive bribery problem introduced by ~\cite{faliszewski2006complexity,faliszewski2009hard}. However, to the best of our knowledge, there is no prior work on estimating the cost of bribery by sampling votes.

\textbf{Organization.}~We formally introduce the terminologies in Section\nobreakspace \ref {sec:prelim};
we present the results on sampling complexity lower bounds in Section\nobreakspace \ref {sec:lwb};
we present polynomial time sampling based algorithms in Section\nobreakspace \ref {sec:upbd};  finally, we conclude in Section\nobreakspace \ref {sec:con}.

\section{Preliminaries}\label{sec:prelim}

Let $V=\{\succ_1, \dots, \succ_n\}$ be the set of all \emph{votes} and $C=\{c_1, \dots, c_m\}$ 
the set of all \emph{candidates}. If not mentioned otherwise, $m$ and $n$ denote the number of candidates and the number of voters respectively. Each vote $\succ_i$ is a complete order over the candidates in $C$. 
For example, for the candidate set $C=\{a, b\}$, $a \succ_i b$ means that the vote $\succ_i$ prefers $a$ to $b$. 
We denote the set of all complete orders over $C$ by $\mathcal{L}(C)$. 
Hence, $\mathcal{L}(C)^n$ denotes the set of all $n$-voters' preference profiles $\succ = (\succ_1, \dots, \succ_n)$. 
A map $r:\cup_{n,|C|\in\mathbb{N}^+}\mathcal{L}(C)^n\longrightarrow 2^C$
is called a \emph{voting rule}. Given a vote profile $\succ \in
\mathcal{L}(C)^n$, we call the candidates in the set $r(\succ)$ the {\em winners}. The pair $(\succ, C)$ is called an $r$--election $\mathcal{E}$ if the voting rule used is $r$.

% Given an election $\mathcal{E}$, we can construct a directed weighted graph $G_{\mathcal{E}}$, called the 
% \textit{weighted majority graph} from $\mathcal{E}$. The set of vertices in $G_{\mathcal{E}}$ is the set of candidates in $\mathcal{E}$. 
% For any two candidates $x$ and $y$, the weight of the edge $(x,y)$ is $D_{\mathcal{E}}(x,y) = N_{\mathcal{E}}(x,y) - N_{\mathcal{E}}(y,x)$, 
% where $N_{\mathcal{E}}(x,y)(N_{\mathcal{E}}(y,x))$ is the number of votes which prefer $x$ to $y$ ($y$ to $x$).

Examples of some common voting rules are as follows.

\textbf{Positional scoring rules:} A collection of vectors $\{\vec{s}_i\}_{i\in\mathbb{N}}$, where $\vec{s}_m=\left(\alpha_1,\alpha_2,\dots,  \alpha_m\right)\in\mathbb{R}^m$ is a $m$-dimensional vector with $\alpha_1\ge\alpha_2\ge\dots\ge\alpha_m$ and $\alpha_1>\alpha_m$ for every $m\in \mathbb{N}$, naturally defines a voting rule -- a candidate gets score $\alpha_i$ from a vote if it is placed at the $i^{th}$ position, and the score of a candidate is the sum of the scores it receives from all the votes. The winners are the candidates with maximum score. Scoring rules remain unchanged if we multiply every $\alpha_i$ by any constant $\lambda>0$ and/or add any constant $\mu$. Hence, we can assume without loss of generality that in every score vector $\vec{\alpha}$, there exists a $j$ with $\alpha_j - \alpha_{j+1}=1$ and $\alpha_i = 0$ for all $i>j$. We call such a vector $\vec{\alpha}$ a normalized score vector. 

 The vector $\vec{\alpha}$ that is $1$ in the first $k$ coordinates and $0$
 elsewhere gives the {\em $k$-approval}  voting rule. $1$-approval is
 called the {\em plurality} voting rule. The score vector $(m-1, m-2,
 \dots, 1, 0)$ gives the {\em Borda} voting rule.
 
 \textbf{Approval:} In approval voting, each vote approves a subset
 of candidates. The winners are the candidates which are approved by the
 maximum number of votes. 
 
 \textbf{Bucklin:} A candidate $x$'s Bucklin score is the minimum number $\ell$ such that at least half 
 of the votes rank $x$ in their top $\ell$ positions. The winners are the candidates with lowest Bucklin score.
 
 \textbf{Maximin:} Given an election $\mathcal{E}$ and any two candidates $x$ and $y$, the quantity $D_{\mathcal{E}}(x,y)$ is defined as $N_{\mathcal{E}}(x,y) - N_{\mathcal{E}}(y,x)$, where $N_{\mathcal{E}}(x,y)$ $(\text{respectively }N_{\mathcal{E}}(y,x))$ is the number of votes which prefer $x$ to $y$ $(\text{respectively }y\text{ to }x).$ The maximin score of a candidate $x$ is $\min_{y\ne x} D_{\mathcal{E}}(x,y)$. The winners are the candidates with maximum maximin score.
 
%  \textbf{Maximin:} The maximin score of a candidate $x$ is $\min_{y\ne x} D_{\mathcal{E}}(x,y)$. The winners are the candidates with maximum maximin score.
 
 \textbf{Copeland$^\alpha$:} The Copeland$^\alpha$ score of a candidate $x$ is $|\{y\ne x:D_{\mathcal{E}}(x,y)>0\}|+\alpha|\{y\ne x:D_{\mathcal{E}}(x,y)=0\}|$, where $\alpha\in [0,1]$. The winners are the candidates with the maximum Copeland$^\alpha$ score. 

% A characteristic vector of a subset $A \subseteq \{1, 2, \ldots, \ell\}$ is a $\ell$ dimensional binary vector $x = (x_1, x_2, \ldots, x_\ell)$ such that $x_i=1$ if and only if $i\in A$.
For score based voting rules (all the voting rules mentioned above are score based), we denote the score of any candidate $x\in C$ by $s(x)$. Given an integer $t$, we denote the set $\{1, \ldots, t\}$ by $[t]$. The notion of margin of victory of an election is defined as follows.
\begin{definition}\label{def:margin of victory}\textsc{(Margin of Victory (MoV))}\\
 Given an election $\mathcal{E}=(\succ, C)$ with voting rule $r$, the margin of victory of $\mathcal{E}$, denoted by $\MOV$, is the minimum number of votes that should be changed to change the winning set $r(\succ)$.
\end{definition}
% Xia \cite{xia2012computing} showed that for most common voting rules
% (including all those mentioned above), when each voter votes
% i.i.d. according to a distribution on the candidates, the margin of
% victory is with high probability, either $\Theta(\sqrt{n})$ or
% $\Theta(n)$. 

\subsection{Chernoff Bound}

We repeatedly use the following concentration inequality:
\begin{theorem}\label{thm:chernoff}
Let $X_1, \dots, X_\ell$ be a sequence of $\ell$ independent
random variables in $[0,1]$ (not necessarily identical). Let $S = \sum_i X_i$ and
let $\mu = \E{S}$. Then, for any $0 \leq \delta \leq 1$: 
$$\Pr[|S - \mu| \geq \delta \mu] < 2\exp(-\delta^2\mu/3)$$
\end{theorem}

\section{Sample Complexity Lower Bounds}\label{sec:lwb}

Our lower bounds for the sample complexity of the \textsc{$(c, \epsilon, \delta)$--MoV} problem are derived from the information-theoretic lower bound for distinguishing two distributions. We start with the following basic observation.
Let $X$ be a random variable taking value $1$ with probability
$\frac{1}{2}-\epsilon$ and $0$ with probability
$\frac{1}{2}+\epsilon$; $Y$ be a random variable taking value $1$ with
probability $\frac{1}{2}$ and $0$ with probability
$\frac{1}{2}$.  Then, it is well-known that every algorithm needs at least
$\frac{1}{4\epsilon^2}\ln \frac{1}{8e\sqrt{\pi}\delta}$ many samples to
distinguish between $X$ and $Y$ with probability of making an error
being at most $\delta$~\cite{canetti1995lower}. Immediately, we have:

\begin{theorem}\label{thm:lb}
 The sample complexity of the \textsc{$(c, \epsilon, \delta)$--MoV} problem for the plurality voting rule is at least $\frac{(1-c)^2}{36\epsilon^2}\ln \left(\frac{1}{8e\sqrt{\pi}\delta}\right)$ for any $c \in [0,1)$.
\end{theorem}

\begin{proof}
Consider two vote distributions $X$ and $Y$, each over the candidate set $\{a,b\}$. In $X$, exactly $\frac{1}{2} +
\frac{6\eps + 2c/n}{1-c}$ fraction of voters prefer $a$ to $b$ and thus the margin of victory is $\frac{3\eps + c/n}{1-c} n$. 
In $Y$, exactly $\frac{1}{2}$ fraction of voters prefer $b$ to $a$ and thus the margin of victory is one. Any \textsc{$(c, \epsilon, \delta)$--MoV} algorithm $\mathcal{A}$ for the plurality voting rule gives us a distinguisher between $X$ and $Y$ with probability of error at most $2\delta$. This is so because, if the input to $\mathcal{A}$ is $X$ then, the output of $\mathcal{A}$ is less than $c+2\eps n$ with probability at most $\delta$, whereas, if the input to $\mathcal{A}$ is $Y$ then, the output of $\mathcal{A}$ is more than $c+\eps n$ with probability at most $\delta$. Now, since $n$ can be arbitrarily large, we get the result.
\end{proof}

\MakeUppercase theorem\nobreakspace \ref {thm:lb} immediately gives the following corollary.
\begin{corollary}\label{cor:lb}
 For any $c \in [0,1)$, every \textsc{$(c, \epsilon, \delta)$--MoV} algorithm needs
 at least $\frac{(1-c)^2}{36\epsilon^2}\ln \left(\frac{1}{8e\sqrt{\pi}\delta}\right)$ many samples for
 all voting rules which reduce to the plurality rule for two
 candidates. In particular, the lower bound holds for 
 scoring rules, approval, Bucklin, maximin, and Copeland$^\alpha$ voting rules.
\end{corollary}

We note that the lower bound results in \MakeUppercase theorem\nobreakspace \ref {thm:lb} and \MakeUppercase corollary\nobreakspace \ref {cor:lb} do not assume anything about the sampling strategy or the computational complexity of the estimator.

\section{Sampling Based Algorithms}\label{sec:upbd}
% 
% We now present sampling based polynomial time algorithms for the \textsc{$(c, \epsilon, \delta)$--MoV} problem for various commonly used voting rules. The general strategy is to sample some number of votes uniformly at random with replacement and compute some quantity (this depends on the specific voting rule under consideration) in polynomial time which gives good estimate of the margin of victory of the underlying election. For the interest of space, we only provide a proof sketch for a few results.
% 
% Another natural approach to estimate the margin of victory is to sample a few votes and feed them to the algorithms for finding the margin of victory by Xia~\cite{xia2012computing}. However, it is not clear how to show interesting sample complexity upper bounds for these algorithms. On the other hand, the theoretical analysis of our algorithms are quite simple. Moreover, for all the voting rules studied here, the sample complexity of our algorithms match the lower bound up to a constant factor and thus proving their optimality, when we have a constant number of candidates.

A natural approach for estimating the margin of victory of an election efficiently is to compute the margin of victory of a suitably small number of sampled votes. Certainly, it is not immediate that samples chosen uniformly at random preserve the value of the margin of victory of the original election within some desired factor. Although it may be possible to formulate clever sampling strategies that tie into the margin of victory structure of the election, we will show that uniformly chosen samples are good enough to design algorithms for estimating the margin of victory for the voting rules studied here. Our proposal has the advantage that the sampling component of our algorithms are always easy to implement, and further, there is no compromise on the bounds in the sense that they are optimal for any constant number of candidates. 

Because our samples are chosen uniformly at random, our analysis relies only on the fact that a sufficiently large sample of votes have been drawn. 
Our algorithms involve computing a quantity (which depends on the voting rule under consideration) based on the sampled votes, which we argue to be a suitable estimate of the margin of victory of the original election. This quantity is not necessarily the margin of victory of the sampled votes. For scoring rules, for instance, we will use the sampled votes to estimate candidate scores, and we use the difference between the top two candidate scores (suitably scaled) as the margin of victory estimate. We also establish a relationship between scores and values of the margin of victory to achieve the desired bounds on the estimate. The overall strategy is in a similar spirit for other voting rules as well, although the exact estimates may be different. We now turn to a more detailed description, although some proofs are omitted due to lack of space. 

\subsection{Scoring Rules and Approval Voting Rule}

We begin with the class of scoring rules. Interestingly, the margin of victory of any scoring rule based election can still be estimated quite accurately by sampling only $\frac{12}{\epsilon^2}\ln\frac{2m}{\delta}$ many votes. An important thing to note is that, the sample complexity upper bound is independent of the score vector. Before embarking on the proof of this general result, we prove a structural lemma which will be used crucially in the subsequent proof.
\begin{lemma}\label{lem:scr}
 Let $\alpha = (\alpha_1, \dots, \alpha_m)$ be any normalized score vector (hence, $\alpha_m=0$). If $w$ and $z$ are the candidates that receive highest and second highest score respectively in a $\alpha$--scoring rule election instance $\mathcal{E}=(V, C)$, then,
 $$ \alpha_1(\MOV[\alpha]-1) \le s(w) - s(z) \le 2\alpha_1\MOV[\alpha] $$
\end{lemma}
\begin{proof}
 Let $\MOV[\alpha]$ be the margin of victory of $\mathcal{E}$. We claim that there must be at least $\MOV[\alpha]-1$ many votes $v\in V$ where $w$ is preferred over $z$. Indeed, otherwise, we swap $w$ and $z$ in  all the votes where $w$ is preferred over $z$. This makes $z$ win the election. However, we have changed at most $\MOV[\alpha]-1$ votes only. This contradicts the definition of margin of victory (see \MakeUppercase definition\nobreakspace \ref {def:margin of victory}). Let $v\in V$ be a vote where $w$ is preferred over $z$. Let $\alpha_i$ and $\alpha_j(\le \alpha_i)$ be the scores received by the candidates $w$ and $z$ respectively from the vote $v$. We replace the vote $v$ by $v^{\prime} = z \succ \cdots \succ c$. This vote change reduces the value of $s(w)-s(z)$ by $\alpha_1 + \alpha_i - \alpha_j$ which is at least $\alpha_1$. Hence, $\alpha_1(\MOV[\alpha]-1) \le s(w) - s(z)$. Each vote change reduces the value of $s(w)-s(z)$ by at most $2\alpha_1$ since $\alpha_m=0$. Hence, $s(w) - s(z) \le 2\alpha_1\MOV[\alpha]$.
\end{proof}

With \MakeUppercase lemma\nobreakspace \ref {lem:scr} at hand, we show our estimation algorithm for the scoring rules next.
\begin{theorem}\label{thm:scr}
 There is a polynomial time \textsc{$(\frac{1}{3}, \epsilon, \delta)$--MoV} algorithm for the scoring rules with sample complexity $\frac{12}{\epsilon^2}\ln\frac{2m}{\delta}$.
\end{theorem}
\begin{proof}
 Let $\alpha = (\alpha_1, \dots, \alpha_m)$ be any arbitrary normalized score vector and $\mathcal{E} = (V, C)$ an election instance. We sample $\ell$ (the value of $\ell$ will be chosen later) votes uniformly at random from the set of votes with replacement. For a candidate $x$, define a random variable $X_i(x) = \frac{\alpha_i}{\alpha_1}$ if $x$ gets a score of $\alpha_i$ from the  $i${th} sample vote. Define $\bar{s}(x) = \frac{n\alpha_1}{\ell}\sum_{i=1}^\ell X_i(x)$ the estimate of $s(x)$, the score of $x$.  Also define $\eps^{\prime} = \frac{\eps}{2}$. Now, using Chernoff bound (\MakeUppercase theorem\nobreakspace \ref {thm:chernoff}), we have the following.
 $$ \Pr\left[ \left|\bar{s}(x) - s(x)\right| \ge \alpha_1 \epsilon^{\prime} n\right] \le
 2\exp\left(-\frac{\epsilon^{\prime2}  \ell}{3}\right)$$
 We now use the union bound to get the following.
%  Now, using union bound, we have the following.
 \begin{equation}
 \resizebox{0.43\textwidth}{!}{$\Pr[ \exists x\in C, |\bar{s}(x) - s(x)| > \alpha_1 \epsilon^{\prime} n ] \le 2m\exp\left(-\frac{\epsilon^{\prime2}  \ell}{3}\right) \label{eqn:scr}$}
 \end{equation}
%  \begin{align} \Pr[ \exists x\in C, |\bar{s}(x) - s(x)| > \alpha_1 \epsilon^{\prime} n ] \le 2m\exp\left(-\frac{\epsilon^{\prime2}  \ell}{3}\right) \label{eqn:scr} 
%  \end{align}

Define $ \bar{M} \eqdef \frac{ \bar{s}(\bar{w}) - \bar{s}(\bar{z})}{1.5\alpha_1} $ the estimate of the margin of victory of the election $\mathcal{E}$ (and thus the output of the algorithm), where $\bar{w}\in \argmax_{x\in C}\{\bar{s}(x)\}$ and $\bar{z}\in \argmax_{x\in C\setminus\{\bar{w}\}}\{\bar{s}(x)\}$. We claim that, if $\forall x\in C, |\bar{s}(x) - s(x)| \le \epsilon^{\prime} n$, then $|\bar{M}-\MOV[\alpha]| \le \frac{1}{3}\MOV[\alpha] + \eps n$. This can be shown as follows.
 \begin{align*}
  \bar{M} - \MOV[\alpha] &= \frac{ \bar{s}(\bar{w}) - \bar{s}(\bar{z}) }{1.5\alpha_1} - \MOV[\alpha]\\
  &\le \frac{ s(w) - s(z) }{1.5\alpha_1} + \frac{2\eps^{\prime}n}{1.5} - \MOV[\alpha]\\
  &\le \frac{1}{3}\MOV[\alpha] + \eps n
 \end{align*}
 
The second inequality follows from the fact that, $\bar{s}(\bar{w}) \le s(\bar{w})+\eps^\prime n \le s(w)+\eps^\prime n$ and $\bar{s}(\bar{z}) \ge \bar{s}(z) \ge s(z) - \eps^\prime n$. The third inequality follows from \MakeUppercase lemma\nobreakspace \ref {lem:scr}. Similarly, we bound $\MOV[\alpha] - \bar{M}$ as follows.
%  $$ \MOV[\alpha] - \bar{M} = \MOV[\alpha] - \frac{ \bar{s}(w) - \bar{s}(z) }{1.5\alpha_1} \le \frac{1}{3}\MOV[\alpha] + \eps n $$
 \begin{align*}
  \MOV[\alpha] - \bar{M} &= \MOV[\alpha] - \frac{ \bar{s}(w) - \bar{s}(z) }{1.5\alpha_1}\\
  &\le \MOV[\alpha] - \frac{ s(w) - s(z) }{1.5\alpha_1} + \frac{2\eps^{\prime}n}{1.5}\\
  &\le \frac{1}{3}\MOV[\alpha] + \eps n
 \end{align*}
 
This proves the claim. Now, we bound the success probability of the algorithm as follows. Let $A$ be the event that $\forall x\in C, |\bar{s}(x) - s(x)| \le \epsilon^{\prime} n$.
 \begin{eqnarray*}
  && \Pr\left[ |\bar{M}-\MOV[\alpha]| \le \frac{1}{3}\MOV[\alpha] + \eps n \right] \\
  &\ge& \Pr\left[ |\bar{M}-\MOV[\alpha]| \le \frac{1}{3}\MOV[\alpha] + \eps n \middle| A \right]  \Pr[ A ]\\
  &=& \Pr[ A ]\\ %\ge 1 - 2m\exp\left(-\epsilon^{\prime2} \ell/3\right)%\\
  &\ge& 1 - 2m\exp\left(-\epsilon^{\prime2} \ell/3\right)
 \end{eqnarray*}
 
 The third equality follows from \MakeUppercase lemma\nobreakspace \ref {lem:scr} and the fourth inequality follows from inequality\nobreakspace \ref {eqn:scr}. Now, by choosing $ \ell = \frac{12}{\epsilon^2}\ln\frac{2m}{\delta}$, we get a \textsc{$(\frac{1}{3}, \epsilon, \delta)$--MoV} algorithm for the scoring rules.
\end{proof}

Now, we show an algorithm for the \textsc{$(0, \epsilon, \delta)$--MoV} problem for the $k$-approval voting rule which not only provides more accurate estimate of the margin of victory, but also has a lower sample complexity. The following lemmas will be used subsequently.
\begin{lemma}\label{lem:kapp}
 Let $\mathcal{E}=(V,C)$ be an arbitrary instance of a \mbox{$k$-approval} election. If $w$ and $z$ are the candidates that receive highest and second highest score respectively in $\mathcal{E}$, then,
 $$ 2(\MOV[k-approval] - 1) < s(w) - s(z) \le 2\MOV[k-approval] $$
\end{lemma}
\begin{proof}
We call a vote $v\in V$ {\em favorable} if $w$ appears within the top $k$ positions and $z$ does not appear within top the $k$ positions in $v$. We claim that the number of favorable votes must be at least $\MOV[k-approval]$. Indeed, otherwise, we swap the positions of $w$ and $z$ in all the favorable votes while keeping the other candidates fixed. This makes the score of $z$ at least as much as the score of $w$ which contradicts the fact that the margin of victory is $\MOV[k-approval]$. Now, notice that the score of $z$ must remain less than the score of $w$ even if we swap the positions of $w$ and $z$ in $\MOV[k-approval]-1$ many favorable votes, since the margin of victory is $\MOV[k-approval]$. Each such vote change increases the score of $z$ by one and reduces the score of $w$ by one. Hence, $2(\MOV[k-approval]-1) < s(w) - s(z)$. 
Again, since the margin of victory is $\MOV[k-approval]$, there exists a candidate $x$ other than $w$ and $\MOV[k-approval]$ many votes in $V$ which can be modified such that $x$ becomes a winner of the modified election. Now, each vote change can reduce the score of $w$ by at most one and increase the score of $x$ by at most one. Hence, $ s(w) - s(x) \le 2\MOV[k-approval] $ and thus $ s(w) - s(z) \le 2\MOV[k-approval] $ since $s(z) \ge s(x)$.
\end{proof}
% \indent For the interest of space, we omit the proof of the following Lemma which can be proved using elementary calculus.

\begin{lemma}\label{lem:funmax}
 Let $f : \mathbb{R} \longrightarrow \mathbb{R}$ be a function defined by $f(t) = e^{-\frac{\lambda}{t}}$. Then, 
 \[  f(x) + f(y) \le f(x+y), \text{ for } x,y > 0, \frac{\lambda}{x+y} > 2, x < y \]
\end{lemma}

% \begin{proof_sketch}
% %  Follows from the fact that $f^{\prime}(t)$ is increasing for $t\in(0, \frac{\lambda}{2})$.
% Follows from elementary calculus.
% \end{proof_sketch}
\begin{proof} For the function $f(x)$, we have following.
$$ f(x) = e^{-\frac{\lambda}{x}} \Rightarrow f^{\prime\prime}(x) = \frac{\lambda^2}{x^4} e^{-\frac{\lambda}{x}} - \frac{2\lambda}{x^3} e^{-\frac{\lambda}{x}}$$
%  \begin{eqnarray*}
%   f(x) &=& e^{-\frac{\lambda}{x}} \\
% %   \Rightarrow f^{\prime}(x) &=& \frac{\lambda}{x^2} e^{-\frac{\lambda}{x}}\\
%   \Rightarrow f^{\prime\prime}(x) &=& \frac{\lambda^2}{x^4} e^{-\frac{\lambda}{x}} - \frac{2\lambda}{x^3} e^{-\frac{\lambda}{x}}
%  \end{eqnarray*}
 Hence, for $y> x>0$ and $\frac{\lambda}{x+y} > 2$, we have $f^{\prime\prime}(x), f^{\prime\prime}(y), f^{\prime\prime}(x+y) > 0$. This implies the following for an infinitesimal positive $\delta$.
 \begin{eqnarray*}
  f^{\prime}(x) &\le& f^{\prime}(y)\\
  \Rightarrow \frac{f(x-\delta) - f(x)}{\delta} &\ge& \frac{ f(y) - f(y-\delta)}{\delta} \\
  \Rightarrow f(x) + f(y) &\le& f(x-\delta) + f(y+\delta)\\
  \Rightarrow f(x) + f(y) &\le& f(x+y)
 \end{eqnarray*}
\end{proof}

With \MakeUppercase \bf Lemma\nobreakspace \ref {lem:kapp} and\nobreakspace  \ref {lem:funmax} at hand, we now describe our margin of victory estimator.
\begin{theorem}\label{thm:kapp}
 There is a polynomial time \textsc{$(0, \epsilon, \delta)$--MoV} algorithm for the $k$-approval rule whose sample complexity is $\frac{12}{\epsilon^2}\ln\frac{2k}{\delta}$.
\end{theorem}

\begin{proof}
Let $\mathcal{E} = (V, C)$ be an arbitrary $k$-approval election. We sample $\ell$ votes uniformly at random from $V$ with
replacement. For a candidate $x$, define a random variable $X_i(x)$ which takes value $1$ if $x$ appears among the top $k$ candidates in the $i^{th}$
sample vote, and $0$ otherwise. Define $\bar{s}(x) \eqdef \frac{n}{\ell}\sum_{i=1}^{\ell}X_i(x)$ the estimate of the score of the candidate $x$,
and let $s(x)$ be the actual score of $x$. Also define $\eps^{\prime} = \frac{\eps}{2}$. Then by the Chernoff bound
(\MakeUppercase theorem\nobreakspace \ref {thm:chernoff}), we have:
 $$ \Pr\left[ |\bar{s}(x) - s(x)| > \epsilon^{\prime} n \right] \le
 2\exp\left(-\frac{\epsilon^{\prime2} \ell n}{3 s(x)}\right)$$
%  By union bound, we have the following.
Now, we apply the union bound to get the following.
 \begin{align*}
  & \Pr[ \exists x\in C, |\bar{s}(x) - s(x)| > \epsilon^{\prime} n ]\\
  &\le \sum_{x\in C} 2\exp\left(-\frac{\epsilon^{\prime2} \ell n}{3s(x)}\right)\\ %\le 2k\exp\left(-\epsilon^{\prime2} \ell/3\right)\numberthis \label{eqn:kapp}% \\
  &\le 2k\exp\left(-\epsilon^{\prime2} \ell/3\right)\numberthis \label{eqn:kapp}
 \end{align*}
 The second inequality follows from \MakeUppercase lemma\nobreakspace \ref {lem:funmax} : The expression $\sum_{x\in C} 2\exp\left(-\frac{\epsilon^{\prime2} \ell n}{3
 s(x)}\right)$ is maximized subject to the constraints that $ 0\le s(x)\le n, \forall x\in C$ and $ \sum_{x\in C} s(x) = kn $, when $ s(x) = n \forall x\in C^{\prime} $ for any subset of candidates $C^{\prime}\subseteq C$ with $|C^{\prime}|=k$ and $s(x)=0 \forall x\in C\setminus C^{\prime}$. 
 
 Now, to estimate the margin of victory of the given election $\mathcal{E}$, let $\bar{w}$ and $\bar{z}$ be candidates with maximum and second maximum estimated score respectively. That is, $\bar{w} \in \argmax_{x\in C} \{ \bar{s}(x) \} \text{ and } \bar{z} \in \argmax_{x\in C\setminus \{\bar{w}\}} \{ \bar{s}(x) \}$. We define $ \bar{M} \eqdef \frac{\bar{s}(\bar{w}) - \bar{s}(\bar{z})}{2} $ the estimate of the margin of victory of the election $\mathcal{E}$ (and thus the output of the algorithm). Let $A$ be the event that $\forall x\in C, |\bar{s}(x) - s(x)| \le \epsilon^{\prime} n$. 
 We bound the success probability of the algorithm as follows.
 \begin{eqnarray*}
  && \Pr\left[ |\bar{M}-\MOV[k-approval]| \le \eps n \right] \\
  &\ge& \Pr\left[ |\bar{M}-\MOV[k-approval]| \le \eps n \middle| A \right] \Pr[ A ]\\
  &=& \Pr[ A ] \\%\ge 1 - 2k\exp\left(-\epsilon^{\prime2} \ell/3\right)%\\
  &\ge& 1 - 2k\exp\left(-\epsilon^{\prime2} \ell/3\right)
 \end{eqnarray*} 
 The second equality follows from \MakeUppercase lemma\nobreakspace \ref {lem:kapp} and an argument analogous to the proof of \MakeUppercase theorem\nobreakspace \ref {thm:scr}. The third inequality follows from inequality\nobreakspace \ref {eqn:kapp}. Now, by choosing $ \ell = \frac{12}{\epsilon^2}\ln\frac{2k}{\delta}$, we get a \textsc{$(0, \epsilon, \delta)$--MoV} algorithm.
\end{proof}

Note that, the sample complexity upper bound matches with the lower bound proved in \MakeUppercase corollary\nobreakspace \ref {cor:lb} for the $k$-approval voting rule when $k$ is a constant, irrespective of the number of candidates. 
% Next, we estimate the margin of victory on to the approval voting rule.
For the approval voting rule, we have the following result.
\begin{theorem}\label{thm:app}
 There is a polynomial time \textsc{$(0, \epsilon, \delta)$--MoV} algorithm for the approval rule with sample complexity $\frac{12}{\epsilon^2}\ln\frac{2m}{\delta}$.
\end{theorem}
\begin{proof_sketch}
 We estimate the approval score of every candidate within an additive factor of $\frac{\eps}{2} n$ by sampling $\frac{12}{\epsilon^2}\ln\frac{2m}{\delta}$ many votes uniformly at random with replacement and the result follows from an argument analogous to the proofs of \MakeUppercase lemma\nobreakspace \ref {lem:kapp} and \MakeUppercase theorem\nobreakspace \ref {thm:kapp}.
\end{proof_sketch}

\subsection{Bucklin Voting Rule}

Now, we consider the Bucklin voting rule. Given an election $\mathcal{E}=(V,C)$, a candidate $x\in C$, and an integer $\ell \in [m]$, we denote the number of votes in $V$ in which $x$ appears within the top $\ell$ positions by $n_\ell(x)$. We prove useful bounds on the margin of victory of any Bucklin election in \MakeUppercase lemma\nobreakspace \ref {lem:bucklin}.
\begin{lemma}\label{lem:bucklin}
 Let $\mathcal{E}=(V,C)$ be an arbitrary instance of a Bucklin election and $w$ be the winner of $\mathcal{E}$. Let us define a quantity $\Delta(\mathcal{E})$ as follows.
 $$ \Delta(\mathcal{E}) \eqdef \min_{\substack{\ell \in [m-1] : n_\ell(w) > n/2,\\ x\in C\setminus\{w\} : n_\ell(x) \le n/2}} \{ n_\ell(w) - n_\ell(x) +1 \} $$
 Then,
 $$ \frac{\Delta(\mathcal{E})}{2} \le \MOV[Bucklin] \le \Delta(\mathcal{E}) $$
\end{lemma}
\begin{proof}
 Pick any $\ell\in[m-1]$ and $x\in C\setminus\{w\}$ such that, $n_\ell(w)>n/2$ and $n_\ell(x)\le n/2$. Now by changing $n_\ell(w)-\floor{n/2}$ many votes, we can ensure that $w$ is not placed within the top $\ell$ positions in more than $n/2$ votes: choose $n_\ell(w)-\floor{n/2}$ many votes where $w$ appears within top $\ell$ positions and swap $w$ with candidates placed at the last position in those votes. Similarly, by changing $\floor{n/2}+1-n_\ell(x)$ many votes, we can ensure that $x$ is placed within top $\ell$ positions in more than $n/2$ votes. Hence, by changing at most $n_\ell(w)-\floor{n/2}+\floor{n/2}+1-n_\ell(x) = n_\ell(w) - n_\ell(x) +1$ many votes, we can make $w$ not win the election. Hence, $\MOV[Bucklin]\le n_\ell(w) - n_\ell(x) +1$. Now, since we have picked an arbitrary $\ell$ and an arbitrary candidate $x$, we have $\MOV[Bucklin] \le \Delta(\mathcal{E})$.
 
 For the other inequality, since the margin of victory is $\MOV[Bucklin]$, there exists an $\ell^\prime\in[m-1]$, a candidate $x\in C\setminus\{w\}$, and $\MOV[Bucklin]$ many votes in $V$ such that, we can change those votes in such a way that in the modified election, $w$ is not placed within top $\ell^\prime$ positions in more than $n/2$ votes and $x$ is placed within top $\ell^\prime$ positions in more than $n/2$ votes. Hence, we have the following.
 
$$ \MOV[Bucklin] \ge n_\ell^\prime(w)-\floor*{\frac{n}{2}}, \MOV[Bucklin]\ge \floor*{\frac{n}{2}}+1-n_\ell^\prime(x) $$

% \resizebox{.98\linewidth}{!}{
%   \begin{minipage}{\linewidth}
 \begin{align*}
  \Rightarrow \MOV[Bucklin] &\ge \max\{n_{\ell^\prime}(w)-\floor*{\frac{n}{2}}, \floor*{\frac{n}{2}}+1-n_{\ell^\prime}(x)\}\\ 
  \Rightarrow  \MOV[Bucklin] &\ge\frac{n_{\ell^\prime}(w)-\floor*{\frac{n}{2}}+\floor*{\frac{n}{2}}+1-n_{\ell^\prime}(x)}{2} \\
  &\ge \frac{\Delta(\mathcal{E})}{2}
%   & = \frac{n_\ell(w) - n_\ell(x) +1}{2}\\
 \end{align*}
% \end{minipage}
% }
%  \begin{equation*}
%   \resizebox{0.5\textwidth}{!}{$\Rightarrow  \MOV[Bucklin] \ge\frac{n_{\ell^\prime}(w)-\floor*{\frac{n}{2}}+\floor*{\frac{n}{2}}+1-n_{\ell^\prime}(x)}{2} \ge \frac{\Delta(\mathcal{E})}{2}$}
%  \end{equation*}
\end{proof}

Notice that, given an election $\mathcal{E}$, $\Delta(\mathcal{E})$ can be computed in polynomial amount of time. \MakeUppercase lemma\nobreakspace \ref {lem:bucklin} leads us to the following Theorem.

\begin{theorem}\label{thm:bucklin}
 There is a polynomial time \textsc{$(\frac{1}{3}, \epsilon, \delta)$--MoV} algorithm for the Bucklin rule with sample complexity $\frac{12}{\epsilon^2}\ln\frac{2m}{\delta}$.
\end{theorem}

\begin{proof_sketch}
 Similar to the proof of \MakeUppercase theorem\nobreakspace \ref {thm:kapp}, we estimate, for every candidate $x\in C$ and for every integer $\ell\in[m]$, the number of votes where $x$ appears within top $\ell$ positions within an approximation factor of $(0, \frac{\eps}{2})$. Next, we compute an estimate of $\bar{\Delta}(\mathcal{E})$ from the sampled votes and output the estimate for the margin of victory as $\bar{\Delta}(\mathcal{E})/1.5$. Using \MakeUppercase lemma\nobreakspace \ref {lem:bucklin}, we can argue the rest of the proof in a way that is analogous to the proofs of \MakeUppercase \bf Theorem\nobreakspace \ref {thm:scr} and\nobreakspace  \ref {thm:kapp}.
\end{proof_sketch}

\subsection{Maximin Voting Rule}

Next, we show the result for the maximin voting rule.

\begin{lemma}\label{lem:maximin}
 Let $\mathcal{E}=(V,C)$ be any instance of a maximin election. If $w$ and $z$ are the candidates that receive highest and second highest maximin score respectively in $\mathcal{E}$, then,
 $$ 2\MOV[maximin] \le s(w) - s(z) \le 4\MOV[maximin] $$
\end{lemma}

\begin{proof}
 Each vote change can increase the value of $s(z)$ by at most two and decrease the value of $s(w)$ by at most two. Hence, we have $s(w) - s(z) \le 4\MOV[maximin]$. Let $x$ be the candidate that minimizes $D_{\mathcal{E}}(w,x)$, that is, $x\in \argmin_{x\in C\setminus \{w\}}\{D_{\mathcal{E}}(w,x)\}$. Let $v\in V$ be a vote where $w$ is preferred over $x$. We replace the vote $v$ by the vote $v^{\prime} = z \succ x \succ \cdots \succ w$. This vote change reduces the score of $w$ by two and does not reduce the score of $z$. Hence, $s(w) - s(z) \ge 2\MOV[maximin]$.
\end{proof}

\begin{theorem}\label{thm:maximin}
 There is a polynomial time \textsc{$(\frac{1}{3}, \epsilon, \delta)$--MoV} algorithm for the maximin rule with sample complexity $\frac{24}{\epsilon^2}\ln\frac{2m}{\delta}$.
\end{theorem}

\begin{proof_sketch}
 Let $\mathcal{E} = (V, C)$ be an instance of maximin election. Let $x$ and $y$ be any two candidates. We sample $\ell$
 votes uniformly at random from the set of all votes with replacement. 
%  Let $X_i(x,y)$ be a random variable taking value $1$ if $x\succ y$ in the $i^{th}$ sample vote and $-1$ otherwise.
%  Let $X_i(x,y)$ be a random variable  defined as follows.
 $$ X_i(x,y) = \begin{cases}
           1,& \text{if } x\succ y \text{ in the } i^{th} \text{ sample vote}\\
           -1,& \text{else}
          \end{cases}
 $$
 Define $\bar{D_{\mathcal{E}}}(x,y) = \frac{n}{\ell}\sum_{i=1}^{\ell}X_i(x,y)$. By using the Chernoff bound and union bound, we have the following.
 \begin{equation*}
 \resizebox{0.48\textwidth}{!}{$\Pr\left[ \exists x,y\in C, |\bar{D_{\mathcal{E}}}(x,y) - D_{\mathcal{E}}(x,y)| > \epsilon n \right] \le 2m^2\exp\left(-\frac{\epsilon^2  \ell}{3}\right)$}
 \end{equation*}
%  We estimate $\bar{D_{\mathcal{E}}}(x,y)$ within an additive error of $\epsilon n$ for every candidates $x, y\in C$. 
 We define $\bar{M} \eqdef \frac{\bar{s}(\bar{w})-\bar{s}(\bar{z})}{3}$, the estimate of the margin of victory of $\mathcal{E}$, where 
%  $\bar{w}$ and $\bar{z}$ are the candidates that receive highest and second highest maximin score in the sampled votes.
 $\bar{w}\in \argmax_{x\in C}\{\bar{s}(x)\}$ and $\bar{z}\in \argmax_{x\in C\setminus\{\bar{w}\}}\{\bar{s}(x)\}$. 
 Now, using \MakeUppercase lemma\nobreakspace \ref {lem:maximin}, we can complete the rest of the proof in a way that is analogous to the proof of \MakeUppercase theorem\nobreakspace \ref {thm:scr}.
\end{proof_sketch}

\subsection{Copeland$^\alpha$ Voting Rule}

Now, we present our result for the Copeland$^\alpha$ voting rule. Xia introduced the brilliant quantity called the {\em relative margin of victory} (see Section 5.1 in \cite{xia2012computing}) which is a crucial ingredient in our algorithm for the Copeland$^\alpha$ voting rule. Given an election $\election$, a candidate $x\in C$, and an integer (may be negative also) $t$, $s^\prime_t(V, x)$ is defined as follows. 
\begin{align*}
 s^\prime_t(V, x) =& |\{ y\in C: y\ne x, D_{\mathcal{E}}(y,x)<2t \}| \\&+ \alpha|\{ y\in C: y\ne x, D_{\mathcal{E}}(y,x)=2t \}|
\end{align*}
For every two distinct candidates $x$ and $y$, the relative margin of victory, denoted by $RM(x,y)$, between $x$ and $y$ is defined as the minimum integer $t$ such that, $s^\prime_{-t}(V, x) \le s^\prime_t(V, y)$. Let $w$ be the winner of the election $\mathcal{E}$. We define a quantity $\Gamma(\mathcal{E})$ to be $\min_{x\in C\setminus\{w\}} \{RM(w,x)\}$. Notice that, given an election $\mathcal{E}$, $\Gamma(\mathcal{E})$ can be computed in a polynomial amount of time. Now we have the following lemma.

\begin{lemma}\label{lem:copeland}
 $\Gamma(\mathcal{E}) \le \MOV[Copeland^\alpha] \le 2(\ceil*{\log m} +1)\Gamma(\mathcal{E}).$
\end{lemma}

\begin{proof}
 Follows from Theorem 11 in \cite{xia2012computing}.
\end{proof}

\begin{theorem}\label{thm:copeland}
 For the Copeland$^\alpha$ voting rule, there is a polynomial time \textsc{$\left(1-O\left(\frac{1}{\log m}\right), \epsilon, \delta\right)$--MoV} algorithm whose sample complexity is $\frac{96}{\epsilon^2}\ln\frac{2m}{\delta}$.
\end{theorem}

\begin{proof}
 Let $\mathcal{E} = (V, C)$ be an instance of a Copeland$^\alpha$ election. For every $x, y\in C$, we compute $\bar{D_{\mathcal{E}}}(x, y)$, which is an estimate of $D_{\mathcal{E}}(x, y)$, within an approximation factor of $(0, \eps^\prime)$, where $\eps^\prime = \frac{\eps}{4}$. This can be achieved with an error probability at most $\delta$ by sampling $\frac{96}{\epsilon^2}\ln\frac{2m}{\delta}$ many votes uniformly at random with replacement (the argument is same as the proof of \MakeUppercase theorem\nobreakspace \ref {thm:scr}). We define $\bar{s}^\prime_t(V, x) = |\{ y\in C: y\ne x, D_{\mathcal{E}}(y,x)<2t \}| + \alpha|\{ y\in C: y\ne x, D_{\mathcal{E}}(y,x)=2t \}|$. 
 We also define $\overline{RM}(x,y)$ between $x$ and $y$ to be the minimum integer $t$ such that, $\bar{s}^\prime_{-t}(V, x) \le s^\prime_t(V, y)$. 
 Let $\bar{w}$ be the winner of the sampled election, $\bar{z} = \argmin_{x\in C\setminus\{\bar{w}\}} \{\overline{RM}(w,x)\}$, $w$ the winner of $\mathcal{E}$, and $z = \argmin_{x\in C\setminus\{w\}} \{RM(w,x)\}$. Since, $\bar{D_{\mathcal{E}}}(x, y)$ is an approximation of $D_{\mathcal{E}}(x, y)$ within a factor of $(0, \eps^\prime)$, we have the following for every candidate $x, y\in C$.
 $$ s^\prime_t(V, x) -\eps^\prime n \le \bar{s}^\prime_t(V, x) \le s^\prime_t(V, x) +\eps^\prime n$$ 
 \begin{align*}
 RM(x,y) - 2\eps^\prime n \le \overline{RM}(x,y) \le RM(x,y) + 2\eps^\prime n \numberthis \label{eqn:copeland}
 \end{align*}
 Define $\bar{\Gamma}(\mathcal{E}) = \overline{RM}(\bar{w}, \bar{z})$ to be the estimate of $\Gamma(\mathcal{E})$. We show the following claim.
 \begin{claim}\label{claim:copeland}
  With the above definitions of $w, z, \bar{w},$ and $\bar{z}$, we have the following.
  $$ \Gamma(\mathcal{E}) - 4\eps^\prime n \le \bar{\Gamma}(\mathcal{E}) \le \Gamma(\mathcal{E}) + 4\eps^\prime n $$
 \end{claim}
 \begin{proof}
  Below, we show the upper bound for $\bar{\Gamma}(\mathcal{E})$.
  \begin{align*}
   \bar{\Gamma}(\mathcal{E}) = \overline{RM}(\bar{w}, \bar{z}) &\le \overline{RM}(w, \bar{z}) + 2\eps^\prime n\\
   &\le \overline{RM}(w, z) + 2\eps^\prime n\\
   &\le RM(w, z) + 4\eps^\prime n\\
   &= \Gamma(\mathcal{E}) + 4\eps^\prime n
  \end{align*}

  The second inequality follows from the fact that $\bar{D_{\mathcal{E}}}(x,y)$ is an approximation of $D_{\mathcal{E}}(x,y)$ by a factor of $(0, \eps^\prime)$. The third inequality follows from the definition of $\bar{z}$, and the fourth inequality uses inequality\nobreakspace \ref {eqn:copeland}. Now, we show the lower bound for $\bar{\Gamma}(\mathcal{E})$.
  \begin{align*}
   \bar{\Gamma}(\mathcal{E}) = \overline{RM}(\bar{w}, \bar{z}) &\ge \overline{RM}(w, \bar{z}) - 2\eps^\prime n\\
   &\ge RM(w,\bar{z}) - 4\eps^\prime n\\
   &\ge RM(w,z) - 4\eps^\prime n\\
   &= \Gamma(\mathcal{E}) - 4\eps^\prime n
  \end{align*}
% \[ \bar{\Gamma}(\mathcal{E}) = \overline{RM}(\bar{w}, \bar{z}) \ge \Gamma(\mathcal{E}) - 4\eps^\prime n \]

The third inequality follows from inequality\nobreakspace \ref {eqn:copeland} and the fourth inequality follows from the definition of $z$.
 \end{proof}

 We define $\bar{M}$, the estimate of $\MOV[Copeland^\alpha]$, to be $\frac{4(\log m + 1)}{2\log m + 3}\bar{\Gamma}(\mathcal{E})$. The following argument shows that $\bar{M}$ is a $\left(1-O\left(\frac{1}{\log m}\right), \epsilon, \delta\right)$--estimate of $\MOV[Copeland^\alpha]$.
% \scalebox{2}{\parbox{0.1\linewidth}{%
% \resizebox{.95\linewidth}{!}{
%   \begin{minipage}{\linewidth}
 \begin{align*}
  &\bar{M} - \MOV[Copeland^\alpha]\\ 
  &= \frac{4(\log m + 1)}{2\log m + 3}\bar{\Gamma}(\mathcal{E}) - \MOV[Copeland^\alpha]\\
  &\le \frac{4(\log m + 1)}{2\log m + 3}\Gamma(\mathcal{E}) - \MOV[Copeland^\alpha] + \frac{16(\log m + 1)}{2\log m + 3}\eps^\prime n\\
  &\le \frac{4(\log m + 1)}{2\log m + 3}\MOV[Copeland^\alpha] - \MOV[Copeland^\alpha] + \eps n\\
  &\le \frac{2\log m +1}{2\log m +3}\MOV[Copeland^\alpha] + \eps n\\
  &\le \left(1-O\left(\frac{1}{\log m}\right) \right)\MOV[Copeland^\alpha] + \eps n
 \end{align*}
% \end{minipage}
% }
% }}
 The second inequality follows from \MakeUppercase claim\nobreakspace \ref {claim:copeland} and the  third inequality follows from \MakeUppercase lemma\nobreakspace \ref {lem:copeland}. Analogously, we have:
 
%  \begin{equation*}
%  \resizebox{0.4\textwidth}{!}{$\MOV[Copeland^\alpha]-\bar{M} \le \left(1-O\left(\frac{1}{\log m}\right) \right)\MOV[Copeland^\alpha] + \eps n$\qedhere}
%  \end{equation*}
% \[
% \pushQED{\qed} 
% \resizebox{0.43\textwidth}{!}{$\MOV[Copeland^\alpha]-\bar{M} \le \left(1-O\left(\frac{1}{\log m}\right) \right)\MOV[Copeland^\alpha] + \eps n$}\qedhere
% \popQED
% \] 
%  $$ \MOV[Copeland^\alpha]-\bar{M} \le \left(1-O\left(\frac{1}{\log m}\right) \right)\MOV[Copeland^\alpha] + \eps n $$
\resizebox{.95\linewidth}{!}{
  \begin{minipage}{\linewidth}
 \begin{align*} 
  &\MOV[Copeland^\alpha]-\bar{M}\\
  &= \MOV[Copeland^\alpha] - \frac{4(\log m + 1)}{2\log m + 3}\bar{\Gamma}(\mathcal{E})\\
  &\le \MOV[Copeland^\alpha] - \frac{4(\log m + 1)}{2\log m + 3}\Gamma(\mathcal{E}) + \frac{16(\log m + 1)}{2\log m + 3}\eps^\prime n\\
  &\le \MOV[Copeland^\alpha] - \frac{2(\log m + 1)}{2\log m + 3}\MOV[Copeland^\alpha] + \eps n\\
  &\le \frac{2\log m +1}{2\log m +3}\MOV[Copeland^\alpha] + \eps n\\
  &\le \left(1-O\left(\frac{1}{\log m}\right) \right)\MOV[Copeland^\alpha] + \eps n
 \end{align*}
\end{minipage}
}
 The second line follows \MakeUppercase claim\nobreakspace \ref {claim:copeland} and the  third line follows from \MakeUppercase lemma\nobreakspace \ref {lem:copeland}.
\end{proof}

The approximation factor in \MakeUppercase theorem\nobreakspace \ref {thm:copeland} is weak when we have a large number of candidates. The main difficulty for showing a better approximation factor for the Copeland$^\alpha$ voting rule is to find a polynomial time computable quantity (for example, $\Gamma(\mathcal{E})$ in \MakeUppercase lemma\nobreakspace \ref {lem:copeland}) that exhibits tight bounds with margin of victory. We remark that, existence of such a quantity will not only imply a better estimation algorithm, but also, a better approximation algorithm (the best known approximation factor for finding the margin of victory for the Copeland$^\alpha$ voting rule is $O(\log m)$ and it uses the quantity $\Gamma(\mathcal{E})$). 
However, we remark that \MakeUppercase theorem\nobreakspace \ref {thm:copeland} will be useful in applications, for example, post election audit and polling, where the number of candidates is often small.
% \vspace{-0.3ex}
\section{Conclusion}\label{sec:con}%\vspace{-0.8ex}
We have introduced the \textsc{$(c, \epsilon, \delta)$--MoV} problem and presented efficient sampling based algorithms for solving it for many commonly used voting rules. Besides closing the gap in the sample complexity, an interesting future direction is to study how the knowledge of social network structure among the voters impacts sample complexity. 
Characterizing voting rules for which the sample complexity of this problem is independent of $m$ and $n$ is another interesting research direction to pursue.

\section{Acknowledgement}

We thank Neeldhara Misra for all the active and helpful discussions.
% \newpage

\bibliographystyle{apalike}
\bibliography{mov}

\end{document}